\documentclass[runningheads]{llncs}
\usepackage{graphicx}
\usepackage{float}
\usepackage{amsmath,amssymb}
\usepackage[ruled,vlined]{algorithm2e}
\usepackage{booktabs}
\usepackage{multirow}
\usepackage{array}
\usepackage{xcolor}
\usepackage[most]{tcolorbox}
\usepackage{hyperref}

\newcommand{\BX}{\mathbf{X}}
\newcommand{\BY}{\mathbf{Y}}

\begin{document}

\title{GRASP: Gradient-Aligned Sequential Parameter Transfer for Memory-Efficient Multi-Source Learning}
\titlerunning{GRASP: Memory-Efficient Multi-Source Transfer}
\author{Mary Isabelle Wisell\inst{1} \and Nicholas Jacobs\inst{2} \and Aayush Manandhar\inst{3} \and Salimeh Yasaei Sekeh\inst{1}}
\authorrunning{M. I. Wisell et al.}
\institute{San Diego State University, San Diego, CA, USA \and University of Utah, Salt Lake City, UT, USA \and University of Maine, Orono, ME, USA}
\maketitle

\begin{abstract}
Multi-source transfer learning faces a fundamental scalability bottleneck: existing approaches require either loading all $K$ source models into memory simultaneously during parameter fusion, requiring $O(K)$ memory, or deploying all models at inference time, making production deployment infeasible. We propose GRASP (Gradient-Aligned Sequential Parameter Transfer), which achieves superior knowledge integration while maintaining $O(1)$ memory consumption through three key innovations: (1) sequential processing that merges one source at a time into an evolving target model, (2) parameter-wise gradient alignment that selectively transfers only parameters whose optimization directions align with the target domain, avoiding negative transfer, and (3) iterative fine-tuning that adapts transferred knowledge before integrating the next source. Extensive experiments across three continual learning benchmarks (Yearbook, CLEAR-10, CLEAR-100) spanning 10 to 108-year temporal distribution shifts and four architectures (1.3M to 25.6M parameters) demonstrate that GRASP achieves 93.5\% mean accuracy over all datasets and architectures compared to ensemble method's 71.7\% accuracy while requiring only constant memory versus $K$ models for standard multi-source fusion. Critically, GRASP's sequential design enables incremental source integration without re-processing previously merged models and scales to arbitrarily many sources without memory growth, making it uniquely suitable for resource-constrained deployment and continually evolving source domains.

\end{abstract}

\section{Introduction}

Transfer learning has revolutionized deep learning by enabling rapid adaptation to new domains through knowledge reuse from pre-trained models~\cite{pan2010survey,weiss2016survey}. While single-source transfer has proven highly effective, real-world applications increasingly involve multiple heterogeneous source domains, each capturing complementary knowledge relevant to the target task. Multi-source transfer learning promises to harness this diverse knowledge for superior target performance; yet, existing approaches face fundamental trade-offs that severely limit their practical deployment.

\textbf{The memory scalability problem:} Current multi-source methods fall into three categories, each with critical limitations. \textit{Ensemble methods}~\cite{dietterich2000ensemble,lakshminarayanan2017simple} maintain $K$ independently trained source models and combine their predictions through weighted averaging. While conceptually simple and embarrassingly parallel, ensembles require $O(K)$ space complexity, making deployment infeasible in memory-constrained environments. \textit{Parameter fusion methods}~\cite{wortsman2022model,matena2022merging,ilharco2023editing} merge source parameters into a single model, eliminating inference overhead but typically requiring all sources loaded simultaneously during merging, creating an $O(K)$ memory bottleneck that prevents scaling to large source collections. \textit{Parameter-efficient methods}~\cite{pfeiffer2021adapterfusion,razdaibiedina2023progressive} train lightweight adapters for each source, but these adapters accumulate across sources and suffer from catastrophic forgetting~\cite{kirkpatrick2017overcoming}.

\textbf{Our approach (GRASP):} We propose Gradient-Aligned Sequential Parameter Transfer (GRASP), which resolves these fundamental limitations through: \textbf{(1) Sequential processing with constant memory:} processing sources one at a time with $O(1)$ memory complexity (only 2 models in memory); \textbf{(2) Gradient-aligned parameter selection:} selectively transferring only parameters whose gradients exhibit positive cosine similarity with target gradients; \textbf{(3) Iterative integration with adaptation:} fine-tuning after each source merge to ensure compatibility with subsequently transferred knowledge.

\textbf{Contributions:} (1) A memory-efficient sequential transfer framework achieving $O(1)$ memory complexity regardless of source count. (2) A gradient alignment criterion providing principled parameter-level selection that identifies beneficial knowledge while avoiding negative transfer. (3) Theoretical analysis establishing bounds on multi-source transfer effectiveness through Fisher Information formulation. (4) Comprehensive empirical validation demonstrating superior accuracy, stability, and memory efficiency across diverse temporal shifts. \textbf{Code:} Available at \url{https://github.com/Sekeh-Lab/grasp-multisource-transfer}.

\section{Related Work}

\textbf{Multi-Source Transfer and Fusion.}
Multi-source domain adaptation leverages multiple sources to improve target performance~\cite{mansour2009domain}. Recent fusion methods combine independently trained models without retraining: model soups~\cite{wortsman2022model} average weights from models with different hyperparameters, Fisher-weighted averaging~\cite{matena2022merging} weights parameters by Fisher information importance, and task arithmetic~\cite{ilharco2023editing} demonstrates that task vectors can be added or negated. While these methods achieve strong performance, they require loading all $K$ sources simultaneously during merging ($O(K)$ memory), necessitate re-merging when adding sources, and employ uniform or coarse-grained weighting schemes.

\textbf{Ensemble Methods.}
Ensembles combine models through voting or averaging~\cite{dietterich2000ensemble}, with deep ensembles~\cite{lakshminarayanan2017simple} providing uncertainty estimates at the cost of maintaining $K$ models in memory and executing $K$ forward passes. Both approaches suffer from linear memory and computational scaling without support for continuous extension.

\textbf{Gradient-Based Transfer.}
Gradient alignment has emerged as an indicator of successful transfer. Du et al.~\cite{du2021gradient} proposed gradient distribution alignment for domain adaptation, while Wang et al.~\cite{wang2024enhancing} introduced prompt gradient alignment for domain-level decisions. Standley et al.~\cite{standley2020which} used gradient cosine similarity for task affinity in multi-task learning. However, prior work applies gradient alignment at the domain or task level rather than for parameter-level selective transfer.

\textbf{Parameter-Efficient and Sequential Methods.}
Parameter-efficient fine-tuning reduces trainable parameters through low-rank decomposition~\cite{hu2021lora}. Sequential methods include AdapterFusion~\cite{pfeiffer2021adapterfusion}, which composes adapters with learned weights, and sequential adapter training~\cite{razdaibiedina2023progressive}, which trains task-specific adapters sequentially but accumulates $K$ adapter modules. Continual learning prevents catastrophic forgetting through regularization~\cite{kirkpatrick2017overcoming} or dynamic architectures~\cite{rusu2016progressive}, but typically requires architectural growth or specialized memory mechanisms.

\textbf{GRASP Distinctions.}
Unlike fusion methods requiring $O(K)$ memory and re-merging, GRASP achieves $O(1)$ memory through sequential processing and continuous integration. Compared to ensembles, GRASP provides single-model inference with selective knowledge aggregation. GRASP extends gradient alignment from domain-level~\cite{du2021gradient,wang2024enhancing} to parameter-level selection, enabling fine-grained transfer control. Unlike sequential adapter methods~\cite{razdaibiedina2023progressive} and continual learning approaches~\cite{rusu2016progressive}, GRASP integrates knowledge into a single model without adapter accumulation or architectural growth.

\section{Methodology and Theoretical Analysis}
\subsection{Problem Formulation}\label{sec3.1}

We consider a sequence of $M$ sources $\{(\BX_m,\BY_m)\}_{m=1,\ldots,M}$, where $\BX_m$ is the domain of source $m$ and $\BY_m$ are class sets of the $m$-th source. The target is denoted $T$, $\{(\BX_T,\BY_T)\}$, with class set $\BY_T$.

\begin{definition}[\textbf{Multi-source TL (MS-TL)}]
For any ground event $\mathcal{D}$ from $M$ sources denoted by $\mathcal{D}^u_M$, the goal of MS-TL is to learn $P(x\in \BX_{T}|\mathcal{D}^u_M)$. We assume source domains are disjoint, i.e., $\BX_m\bigcap \BX_{m'}=\emptyset$, $\forall m\neq m'$ and $\mathcal{D}_M^u=\bigcup_{m=1}^M \BX_m$ and 
\begin{equation}\label{Def:MS-TL}
P(x\in \BX_{T}|\mathcal{D}^u_M)= \sum_{m=1}^M P(x\in \BX_{T}|x\in\BX_m) P(x\in\BX_m).
\end{equation}
Because $\BX_m\bigcap \BX_{m'}=\emptyset$, the definition for a particular source $\BX_m$ is:
\begin{equation}\label{Def:MS-TL-1}
P(x\in \BX_{T}|x\in\BX_m) P(x\in\BX_m), 
\end{equation}
where source prediction (SP) probability is $P(x\in\BX_m)$ and $m$-th source transfer prediction ($m$-STP) probability is $P(x\in \BX_{T}|x\in\BX_m)$.
\end{definition}

\begin{definition}[\textbf{Ensemble TL (E-TL)}]
For a set of sources $\{\BX_m\}_{m=1,\ldots,M}$ (i.e., $\mathcal{D}^e_M= \{\BX_m\}_{m=1}^M$), E-TL learns target $\BX_T$ given $\mathcal{D}^e_M$:
\begin{equation}\label{Def:E-TL}
P(x\in \BX_{T}|\mathcal{D}^e_M)=\sum\limits_{m=1}^M \alpha_m P(x\in \BX_{T}|x\in\BX_m),
\end{equation}
where $\alpha_m\in(0,1)$ is the ensemble weight of source $\BX_m$. The probability $P(x\in \BX_{T}|x\in\BX_m)$ is the $m$-th STP probability.
\end{definition}

\textbf{Remark:} If we set $\alpha_m=P(x\in\BX_m)$, E-TL implies MS-TL.

\begin{definition}[\textbf{Sequential TL (S-TL)}]
For a sequence of sources $\{(\BX_m,\BY_m)\}_{m=1,\ldots,M}$ where $\mathcal{D}^s_M$ is the ground event with sequence $\mathcal{D}^s_M= \BX_1\rightarrow \BX_2\rightarrow \ldots \BX_M$:
\begin{equation}\label{Def:S-TL}
P(x\in \BX_{T}|\mathcal{D}^s_M)= P(x\in \BX_{T}|x\in\BX_{M}).
\end{equation}
\end{definition}

\begin{lemma}\label{lemma1}
Suppose source prediction probabilities are bounded by $\gamma_m$, $m=1,\ldots,M$, i.e., $P(x\in\BX_m)\leq \gamma_m$. Then E-TL and MS-TL prediction for target $\BX_T$ is bounded by:
\begin{align}
\left| P(x\in \BX_{T}|\mathcal{D}^e_M) - P(x\in \BX_{T}|\mathcal{D}^u_M)\right| 
\leq \sum\limits_{m=1}^M \beta_m P(x\in \BX_{T}|x\in\BX_m),
\end{align}
where $\beta_m$, $m=1,\ldots M$, are constants.
\end{lemma}

\subsection{Source Effectiveness and Informativeness}

In S-TL, determining the effectiveness and informativeness of each source for the target is critical.

\begin{definition}[\textbf{Source Effectiveness}]
Let $d(\mathbb{P}_1\|\mathbb{P}_2)$ be a symmetric distance between distributions (e.g., L2 distance, total variation, symmetric KL-divergence). Given consecutive sources $\BX_{m-1}$ and $\BX_m$, the effectiveness of source $\BX_m$ for target $\BX_T$ is:
\begin{equation}\label{def:effectiveness}
\mathcal{E}(\BX_{m-1\rightarrow m}):= d\left( P(x\in \BX_{T}|\mathcal{D}^s_m), P(x\in \BX_{T}|\mathcal{D}^s_{m-1})\right).
\end{equation}
Source $\BX_m$ is $\delta$-effective if $\mathcal{E}(\BX_{m-1\rightarrow m})\geq \delta$.
\end{definition}

\begin{definition}[\textbf{Source Informativeness}]
Given consecutive sources $\BX_{m-1}$ and $\BX_m$, source $\BX_m$ is $\gamma$-informative if:
\begin{align}\label{def.informativeness}
\mathcal{I}(\BX_{m-1\rightarrow m}):=\frac{P(x\in \BX_{T}|\mathcal{D}^s_m)}{P(x\in \BX_{T}|\mathcal{D}^s_{m-1})} \geq \gamma,
\end{align}
where constant $\gamma>1$.
\end{definition}

\textbf{Remark:} When distance function $d$ is absolute value of logarithmic probability difference:
\begin{align}\label{example.Info}
\mathcal{E}(\BX_{m-1\rightarrow m})
&=\big|\log P(x\in \BX_{T}|\mathcal{D}^s_m) - \log P(x\in \BX_{T}|\mathcal{D}^s_{m-1})\big|\nonumber\\
&=\big| \log \left(\frac{P(x\in \BX_{T}|\mathcal{D}^s_m)}{P(x\in \BX_{T}|\mathcal{D}^s_{m-1})}\right)\big|,
\end{align}
we have $\mathcal{E}(\BX_{m-1\rightarrow m})=|\log (\mathcal{I}(\BX_{m-1\rightarrow m}))|$. For $k$ sources:
\begin{equation}\label{effectiveness_k_sources}
\mathcal{E}(\BX_{m\rightarrow m+k}) = \left| \log \left(\frac{P(x \in \BX_{T} | \mathcal{D}^s_{m+k})}{P(x \in \BX_{T} | \mathcal{D}^s_m)} \right)\right|,
\end{equation}
where $\mathcal{D}^s_{m+k}=\BX_{m}\rightarrow\BX_{m+1} \rightarrow \ldots \rightarrow \BX_{m+k}$.

\textbf{Notation:}
\begin{itemize}
\item $P_{\theta^{(m)}_T}:=P(x\in \BX_T|\mathcal{D}^s_m)$: Probability when target is learned over sequence $\mathcal{D}^s_m=\BX_1\rightarrow \BX_2\rightarrow \ldots \BX_m$
\item $P_{\theta^{(m\rightarrow m+1)}_T}:=P(x\in \BX_T|\mathcal{D}^s_m\rightarrow \BX_{m+1})$: Probability when target is learned over $\mathcal{D}_{m+1}$ after learning $\mathcal{D}^s_m$. Note $\mathcal{D}^s_{m+1}=\mathcal{D}^s_m\rightarrow \BX_{m+1}$
\end{itemize}

We denote learned parameters by $\widehat{\theta}^{(m)}_T$ and use convex combination:
\begin{align}\label{eq.convex}
\widehat{\theta}^{(m\rightarrow m+1)}_T = \lambda \; \widehat{\theta}^{(m+1)}_T+(1-\lambda) \; \widehat{\theta}^{(m)}_T, \;\;\; 0\leq \lambda\leq 1.
\end{align}

For $M$ sequential sources:
\begin{align}
\widehat{\theta}^{(1\rightarrow M)}_T = \sum_{m=1}^M \lambda_m\widehat{\theta}^{(m)}_T, \;\; \text{where}\;\; \sum_{m=1}^M \lambda_m=1.
\end{align}

We solve the optimization problem:
\begin{align}
\widehat{\theta}^{(1\rightarrow M)}_T={\arg\max}_{\theta} P_{\theta^{(1\rightarrow M)}_T}(x\in \BX_T|\mathcal{D}^s_M).
\end{align}

The Fisher information matrix (FIM) is:
\begin{align}\label{def:FIM}
\mathbf{F}(\widehat{\theta}^{(m\rightarrow m+1)}_T)
= - \nabla_{\theta_T} \nabla_{\theta_T} \log P_{\theta_T}(x\in \BX_T|\mathcal{D}^s_{m+1})|_{\theta_T=\widehat{\theta}^{(m\rightarrow m+1)}_T}.
\end{align}

Similarly:
\begin{equation}
\mathbf{F}(\widehat{\theta}^{(m)}_T)= - \nabla_{\theta_T} \nabla_{\theta_T} \log P_{\theta_T}(x\in \BX_T|\mathcal{D}^s_{m})|_{\theta_T=\widehat{\theta}^{(m)}_T}.
\end{equation}

\textbf{Assumption 1:} $\mathbf{F}(\widehat{\theta}^{(m\rightarrow m+1)}_T) \succeq \mathbf{F}(\widehat{\theta}^{(m)}_T)$, meaning $\mathbf{F}(\widehat{\theta}^{(m\rightarrow m+1)}_T) - \mathbf{F}(\widehat{\theta}^{(m)}_T)$ is positive semidefinite.

\subsection{Theoretical Bounds}

\paragraph{\textit{Bridging between effectiveness and Fisher information matrix (FIM):}}

\begin{theorem}\label{theorem1}
\textbf{(Bound on informativeness of S-TL)} Under Assumption 1, informativeness is bounded:
\begin{align}\label{bound.Info}
\mathcal{I}(\BX_{m\rightarrow m+1})\leq 
\kappa e^{-\lambda\; \left(\theta_T\right)^T\Xi^{(m:m+1)}},
\end{align}
where $\kappa$ is a constant and 
\begin{align}\label{def.Xi}
\Xi^{(m:m+1)}&:=\mathbf{F}(\widehat{\theta}^{(m\rightarrow m+1)}_T) \Delta^{(m:m+1)},\\
\Delta^{(m:m+1)}&=\widehat{\theta}^{(m)}_T-\widehat{\theta}^{(m+1)}_T.
\end{align}
\end{theorem}

\begin{proof}
We use Taylor expansion of $\log P_{\theta_T}$ around $\widehat{\theta}^{(m\rightarrow m+1)}_T$:
\begin{align}
&\log P_{\theta^{(m)}_T}(x\in \BX_T|\mathcal{D}^s_m) \nonumber\\
&= \log P_{\widehat{\theta}^{(m\rightarrow m+1)}_T}(x\in \BX_T|\mathcal{D}^s_{m+1}) + \nabla_{\theta_T} \log P_{\theta_T}|_{\widehat{\theta}^{(m\rightarrow m+1)}_T}^T (\widehat{\theta}^{(m)}_T - \widehat{\theta}^{(m\rightarrow m+1)}_T)\nonumber\\
&\quad - \frac{1}{2}(\widehat{\theta}^{(m)}_T - \widehat{\theta}^{(m\rightarrow m+1)}_T)^T \mathbf{F}(\widehat{\theta}^{(m\rightarrow m+1)}_T) (\widehat{\theta}^{(m)}_T - \widehat{\theta}^{(m\rightarrow m+1)}_T) + O(\|\Delta\|^3).
\end{align}
Using $\widehat{\theta}^{(m\rightarrow m+1)}_T = \lambda \widehat{\theta}^{(m+1)}_T+(1-\lambda) \widehat{\theta}^{(m)}_T$:
\begin{equation}
\widehat{\theta}^{(m)}_T - \widehat{\theta}^{(m\rightarrow m+1)}_T = \lambda(\widehat{\theta}^{(m)}_T - \widehat{\theta}^{(m+1)}_T) = \lambda \Delta^{(m:m+1)}.
\end{equation}
At optimality, $\nabla_{\theta_T} \log P_{\theta_T}|_{\widehat{\theta}^{(m\rightarrow m+1)}_T} = 0$. Thus:
\begin{align}
\log P_{\theta^{(m)}_T} &\leq \log P_{\widehat{\theta}^{(m\rightarrow m+1)}_T} - \frac{\lambda^2}{2} (\Delta^{(m:m+1)})^T \mathbf{F}(\widehat{\theta}^{(m\rightarrow m+1)}_T) \Delta^{(m:m+1)}.
\end{align}
Similarly for $\log P_{\theta^{(m-1)}_T}$. Taking differences and exponentiating:
\begin{align}
\mathcal{I}(\BX_{m\rightarrow m+1}) = \frac{P_{\theta^{(m)}_T}}{P_{\theta^{(m-1)}_T}} &\leq \kappa \exp\left(-\lambda \theta_T^T \Xi^{(m:m+1)}\right),
\end{align}
where $\Xi^{(m:m+1)} = \mathbf{F}(\widehat{\theta}^{(m\rightarrow m+1)}_T) \Delta^{(m:m+1)}$ and $\kappa$ absorbs constants.
\end{proof}

\textbf{Remark:} Informativeness of source $\BX_{m+1}$ increases if $\lambda$ is larger and $\Delta^{(m:m+1)}$ is negative. If transferring knowledge from source $m+1$ individually implies larger estimated parameters on target compared to source $m$, and S-TL parameters $\hat{\theta}^{(m\rightarrow m+1)}_T$ align more with $\hat{\theta}^{(m+1)}_T$, then source $m+1$ is more informative for target.

\begin{corollary}
\textbf{(Bound on effectiveness of S-TL)} When effectiveness uses absolute logarithmic probability difference, under Assumption 1:
\begin{align}\label{bound.effectiveness}
\mathcal{E}(\BX_{m\rightarrow m+1})\leq |\lambda\; \left(\theta_T\right)^T\Xi^{(m:m+1)}|+\text{constant}.
\end{align}
\end{corollary}

For total effectiveness of $M$ sources under Assumption 1:
\begin{align}\label{bound.effectiveness.M}
\mathcal{E}(\BX_{1\rightarrow M})\leq \big|\sum_{m=1}^{M-1}\lambda\; \left(\theta_T\right)^T\Xi^{(m:m+1)}\big|+\text{constant}.
\end{align}

\begin{theorem}\label{thm2}
\textbf{(Bound on effectiveness of E-TL)} For sequence $\BX_{m\rightarrow m+k}$ = $\BX_{m} \rightarrow \BX_{m+1} \rightarrow \ldots \rightarrow \BX_{m+k}$, effectiveness is:
\begin{equation}
\mathcal{E}(\BX_{m\rightarrow m+k}) = \left| \log \left(\frac{P(x \in \BX_{T} | \mathcal{D}^e_{m+k})}{P(x \in \BX_{T} | \mathcal{D}^e_m)} \right)\right|,
\end{equation}
where $\mathcal{D}^e_{m+k}= \{\mathcal{D}^e_m,\{\BX_{m+i}\}_{i=0}^k\}$. Assuming
$\alpha_{m+k} P(x\in\BX_T|x\in\BX_{m+k})/P(x\in\BX_T|x\in\BX_{m+k})\geq 1$:
\begin{align}\label{bound-E-TL}
\mathcal{E}(\BX_{m\rightarrow m+k}) \geq 
\sum_{i=0}^{k-1}\log (\alpha_{m+i}) +\mathcal{E}(\BX_{m+k-1\rightarrow m+k}).
\end{align}
\end{theorem}

This theorem shows how effectiveness of $k$ sources relates recursively for E-TL. Since E-TL and MS-TL are related, this bound applies to MS-TL as well.

\subsection{GRASP Algorithm}

Figure~\ref{fig:grasp_method} illustrates GRASP's methodology. Algorithm~\ref{alg:grasp} presents the complete procedure.

\begin{figure}[!ht]
\centering
\includegraphics[width=\textwidth]{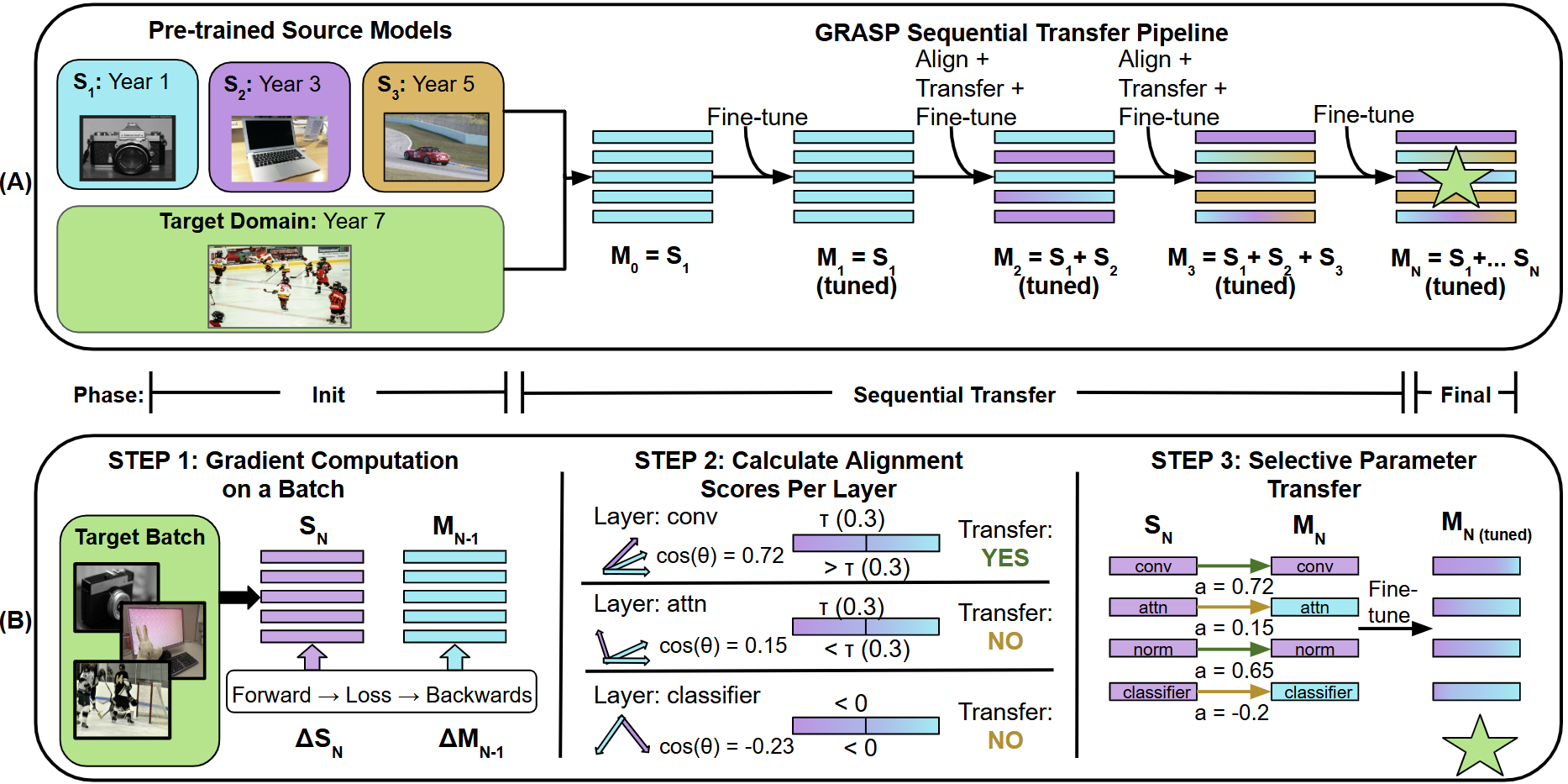}
\caption{GRASP methodology. \textbf{(A)} Sequential transfer pipeline: pre-trained sources progressively merged into single target model via gradient-aligned parameter selection. \textbf{(B)} Three-step process per source: (1) Compute gradients on target batch for merged and source models, (2) Calculate parameter-wise cosine similarity, (3) Selectively transfer aligned parameters (similarity $> \tau$), then fine-tune.}
\label{fig:grasp_method}
\end{figure}

\begin{algorithm}[!ht]
\small
\SetAlgoLined
\textbf{Input:} Source models $\{\theta^*_1, \ldots, \theta^*_K\}$, target dataset $\mathcal{D}_0$, threshold $\tau \in [0,1]$, batch size $B$\;
\textbf{Output:} Merged model $\theta_{\text{merged}}$\;
\BlankLine
Initialize $\theta_{\text{merged}} \leftarrow$ random initialization or $\theta^*_1$\;
\BlankLine
\For{$k = 1, \ldots, K$}{
    \vspace{0.2cm}
    Sample batch $\mathcal{B}_{\text{val}} \sim \mathcal{D}_0$ with $|\mathcal{B}_{\text{val}}| = B$\;
    \vspace{0.3cm}
    \begin{center}
    \begin{tcolorbox}[colback=yellow!5!white,colframe=yellow!75!black,boxsep=0.5pt,left=2pt,right=2pt,top=4pt,bottom=4pt]
    \vspace{-0.05cm}
    \textbf{Stage 1: Compute Target Gradients}\\
    Compute loss $L_0(\theta_{\text{merged}}, \mathcal{B}_{\text{val}})$\\
    Compute gradients $g_{\text{target},j} = \frac{\partial L_0}{\partial \theta_j}\big|_{\theta_{\text{merged}}}$ for all $j$
    \vspace{-0.05cm}
    \end{tcolorbox}
    \vspace{-0.4cm}
    \begin{tcolorbox}[colback=orange!5!white,colframe=orange!75!black,boxsep=0.5pt,left=2pt,right=2pt,top=4pt,bottom=4pt]
    \vspace{-0.05cm}
    \textbf{Stage 2: Compute Source Gradients}\\
    Load source model $\theta^*_k$ into memory\\
    Compute loss $L_0(\theta^*_k, \mathcal{B}_{\text{val}})$ (evaluated on target data)\\
    Compute gradients $g_{k,j} = \frac{\partial L_0}{\partial \theta_j}\big|_{\theta^*_k}$ for all $j$
    \vspace{-0.05cm}
    \end{tcolorbox}
    \vspace{-0.4cm}
    \begin{tcolorbox}[colback=cyan!5!white,colframe=cyan!75!black,boxsep=0.5pt,left=2pt,right=2pt,top=4pt,bottom=4pt]
    \vspace{-0.05cm}
    \textbf{Stage 3: Gradient Alignment Computation}\\
        \For{each parameter $j \in \{1,\ldots,N\}$}{
            $a_{k,j} \leftarrow \frac{g_{k,j} \cdot g_{\text{target},j}}{\|g_{k,j}\| \cdot \|g_{\text{target},j}\|}$
        }
    \vspace{-0.05cm}
    \end{tcolorbox}
    \vspace{-0.4cm}
    \begin{tcolorbox}[colback=red!5!white,colframe=red!75!black,boxsep=0.5pt,left=2pt,right=2pt,top=4pt,bottom=4pt]
    \vspace{-0.05cm}
    \textbf{Stage 4: Selective Parameter Transfer}\\
        \For{each parameter $j \in \{1,\ldots,N\}$}{
            \If{$a_{k,j} > \tau$}{
                $\theta_{\text{merged},j} \leftarrow \theta^*_{k,j}$
            }
            \Else{
                $\theta_{\text{merged},j} \leftarrow \theta_{\text{merged},j}$
            }
        }
    \vspace{-0.05cm}
    \end{tcolorbox}
    \vspace{-0.4cm}
    \begin{tcolorbox}[colback=green!5!white,colframe=green!75!black,boxsep=0.5pt,left=2pt,right=2pt,top=4pt,bottom=4pt]
    \vspace{-0.05cm}
    \textbf{Stage 5: Optional Refinement}\\
    Fine-tune $\theta_{\text{merged}}$ on $\mathcal{D}_0$ for few epochs
    \vspace{-0.05cm}
    \end{tcolorbox}
    \end{center}
    \vspace{0.15cm}
}
\BlankLine
\Return $\theta_{\text{merged}}$
\caption{GRASP: Gradient-Aligned Sequential Parameter Transfer}\label{alg:grasp}
\end{algorithm}

\section{Experimental Setup}

\textbf{Datasets:} We evaluate on three continual learning benchmarks: \textbf{CLEAR-10}~\cite{lin2021clear} features 10 object classes across 11 years (2015-2025). We merge data from years 1-10 into 5 temporal 2-year bins. Each bin serves as target with remaining 4 bins as sources ($\sim$6,000 images per bin and $\sim$30,000 images total). \textbf{CLEAR-100}~\cite{lin2022clear} extends CLEAR-10 with 100 classes. We use a 30-class subset with same temporal structure ($\sim$30,000 images). \textbf{Yearbook}~\cite{ginosar2015century} contains $\sim$38,000 yearbook portraits spanning 108 years (1905-2013). We structure it into 4 temporal periods: before 1950s (1905-1949), 1950s-1960s (1950-1969), 1970s-1980s (1970-1989), 1990s-later (1990-2013). Each period serves as target with remaining 3 as sources.

\textbf{Architectures:} We evaluate on four architectures with varying capacities: MobileViT-XXS~\cite{mehta2022mobilevit} (1.3M parameters) and MobileViT-XS~\cite{mehta2022mobilevit} (2.3M parameters) using Apple's pretrained models; EfficientNet-B1~\cite{tan2019efficientnet} (7.8M parameters) using Google's pretrained model; ResNet-50~\cite{he2016deep} (25.6M parameters) using torchvision's ImageNet-pretrained weights.

\textbf{Baselines:} (1) \textbf{Ensemble}: uniform averaging of $K$ source predictions. (2) \textbf{Multi-Source}: uniform parameter averaging~\cite{wortsman2022model}. (3) \textbf{PEARL}: parameter-efficient adapter-based multi-source composition.

\textbf{Implementation:} PyTorch 2.7, NVIDIA RTX 5080 GPU (16GB), AdamW optimizer (lr=1e-4), batch size=32, gradient alignment threshold $\tau = 0.3$, 3 fine-tuning epochs per source. Gradient alignment batches $\mathcal{B}_\text{val}$ are drawn exclusively from the target training split. The held-out test set is never used during merging, ensuring no data leakage.

\section{Results and Analysis}

We evaluate GRASP across accuracy, computational efficiency, and memory scalability. Tables~\ref{tab:overall},~\ref{tab:yearbook},~\ref{tab:clear10}, and~\ref{tab:clear100} present comprehensive results across different experimental configurations (3 datasets $\times$ 4 architectures $\times$ 4-5 targets $\times$ 4 methods).

\subsection{Accuracy Analysis Across Distribution Shifts}

\textbf{Overall Performance:} Table~\ref{tab:overall} aggregates results. GRASP achieves 93.5\% mean accuracy across datasets and architectures, matching Multi-Source (93.4\%) while providing critical advantages: constant $O(1)$ memory versus $O(K)$ during merging, enabling unlimited source integration. PEARL (75.8\%) and Ensemble (71.7\%) show significantly degraded performance. Most striking is Yearbook: GRASP achieves 92.1\% while Ensemble catastrophically fails at 45.5\%, a 46.6 point gap.

\textbf{Extreme Temporal Shifts - Yearbook (108 years):} Table~\ref{tab:yearbook} shows results across 4 temporal periods spanning 1905--2013. Ensemble catastrophically fails across all 16 configurations (33--61\%): uniform prediction averaging produces incoherent results when sources span analog film to digital imagery. GRASP maintains 89--95\% with only 3.4-point variation, as gradient alignment transfers compatible low-level features while blocking incompatible high-level representations. Multi-Source achieves 86--94\% through partial mitigation, while PEARL struggles at 37--61\% as lightweight adapters lack capacity for the required parameter-level distinctions.

\textbf{Gradual Temporal Drift - CLEAR-10:} Table~\ref{tab:clear10} shows results across 5 temporal bins (2015--2025). GRASP achieves 94.8--97.7\% with remarkable stability; Multi-Source slightly outperforms (97.0\% vs.\ 96.5\%) when sources are mutually compatible, consistent with Theorem~2. ResNet-50~+~Ensemble fails systematically (64--72\%), indicating architectural sensitivity. PEARL struggles at 75.1\%.

\textbf{Fine-Grained Recognition - CLEAR-100:} Table~\ref{tab:clear100} extends to 30 classes. GRASP maintains 87--94\%, matching Multi-Source (both 92.0\%). ResNet-50~+~Ensemble continues to fail (42--45\%) and PEARL degrades to 67.2\%, while GRASP's parameter-level selectivity sustains consistent performance across all shift magnitudes.

\begin{table}[!ht]
\centering
\caption{Overall performance: methods comparison across datasets}
\label{tab:overall}
\small
\begin{tabular}{lccccc}
\toprule
\textbf{Method} & \textbf{Yearbook} & \textbf{CLEAR-10} & \textbf{CLEAR-100} & \textbf{Mean} & \textbf{Memory} \\
\midrule
GRASP (Ours) & \textbf{92.1} & 96.5 & \textbf{92.0} & \textbf{93.5} & O(1) \\
Multi-Source & 90.6 & \textbf{97.0} & \textbf{92.0} & 93.4 & O(K) \\
PEARL & 87.3 & 75.1 & 67.2 & 75.8 & O(K) \\
Ensemble & 45.5 & 89.4 & 80.2 & 71.7 & O(K) \\
\bottomrule
\multicolumn{6}{l}{\footnotesize Mean across 4 architectures per dataset; Memory: complexity in \# sources}
\end{tabular}
\end{table}

\begin{table}[!ht]
\centering
\caption{Yearbook detailed results: Binary classification across 108-year span}
\label{tab:yearbook}
\scriptsize
\begin{tabular}{llcccc}
\toprule
Target & Architecture & GRASP & Multi-Source & Ensemble & PEARL \\
\midrule
\multirow{4}{*}{\shortstack[l]{before\\1950s}} 
& MobileViT-XXS & \textbf{94.0} & 92.5 & 41.6 & 42.0 \\
& MobileViT-XS & \textbf{93.6} & 91.8 & 43.1 & 37.8 \\
& EfficientNet-B1 & \textbf{94.3} & 90.5 & 43.7 & 41.2 \\
& ResNet-50 & \textbf{94.9} & 94.0 & 41.1 & 39.3 \\
\midrule
\multirow{4}{*}{\shortstack[l]{1950s-\\1960s}}
& MobileViT-XXS & \textbf{90.0} & 89.7 & 38.0 & 42.8 \\
& MobileViT-XS & \textbf{91.6} & 90.2 & 36.6 & 46.1 \\
& EfficientNet-B1 & \textbf{92.8} & 87.9 & 33.0 & 45.9 \\
& ResNet-50 & \textbf{93.3} & 91.6 & 41.5 & 39.3 \\
\midrule
\multirow{4}{*}{\shortstack[l]{1970s-\\1980s}}
& MobileViT-XXS & \textbf{91.2} & 89.0 & 51.8 & 61.4 \\
& MobileViT-XS & \textbf{91.4} & 89.3 & 53.2 & 59.0 \\
& EfficientNet-B1 & \textbf{92.6} & 88.9 & 48.5 & 56.2 \\
& ResNet-50 & \textbf{92.5} & 92.2 & 60.5 & 59.7 \\
\midrule
\multirow{4}{*}{\shortstack[l]{1990s-\\later}}
& MobileViT-XXS & \textbf{89.3} & 85.9 & 47.6 & 43.3 \\
& MobileViT-XS & \textbf{90.4} & 87.5 & 50.6 & 47.1 \\
& EfficientNet-B1 & \textbf{91.5} & 86.1 & 49.8 & 49.4 \\
& ResNet-50 & \textbf{90.8} & 90.0 & 47.3 & 48.4 \\
\bottomrule
\end{tabular}
\end{table}

\begin{table}[!ht]
\centering
\caption{CLEAR-10 detailed results: 10-class classification}
\label{tab:clear10}
\scriptsize
\begin{tabular}{llcccc}
\toprule
Target & Architecture & GRASP & Multi-Source & Ensemble & PEARL \\
\midrule
\multirow{4}{*}{year 1-2}
& MobileViT-XXS & 94.9 & \textbf{95.2} & 93.7 & 92.2 \\
& MobileViT-XS & 96.7 & \textbf{97.3} & 96.7 & 95.2 \\
& EfficientNet-B1 & 96.3 & 97.2 & \textbf{98.3} & 95.7 \\
& ResNet-50 & 96.7 & \textbf{97.5} & 67.8 & 92.3 \\
\midrule
\multirow{4}{*}{year 3-4}
& MobileViT-XXS & 94.8 & \textbf{96.5} & 94.7 & 94.0 \\
& MobileViT-XS & 96.8 & \textbf{97.8} & 97.7 & 96.0 \\
& EfficientNet-B1 & 97.5 & \textbf{98.0} & \textbf{98.0} & 97.5 \\
& ResNet-50 & 97.2 & \textbf{97.5} & 64.3 & 90.8 \\
\midrule
\multirow{4}{*}{year 5-6}
& MobileViT-XXS & 95.4 & \textbf{95.7} & 94.2 & 92.2 \\
& MobileViT-XS & 96.8 & \textbf{97.6} & 97.3 & 95.2 \\
& EfficientNet-B1 & 97.2 & 98.3 & \textbf{98.5} & 96.2 \\
& ResNet-50 & 97.0 & \textbf{97.7} & 65.2 & 88.7 \\
\midrule
\multirow{4}{*}{year 7-8}
& MobileViT-XXS & 95.7 & \textbf{96.5} & 93.7 & 92.2 \\
& MobileViT-XS & 95.9 & \textbf{97.3} & 97.0 & 95.7 \\
& EfficientNet-B1 & 96.7 & \textbf{97.2} & \textbf{97.2} & 94.7 \\
& ResNet-50 & 95.6 & \textbf{96.8} & 70.3 & 87.7 \\
\midrule
\multirow{4}{*}{year 9-10}
& MobileViT-XXS & 95.4 & \textbf{96.2} & 95.3 & 93.0 \\
& MobileViT-XS & 97.5 & \textbf{97.7} & 97.5 & 94.5 \\
& EfficientNet-B1 & 97.7 & 98.5 & \textbf{99.0} & 97.5 \\
& ResNet-50 & 97.7 & \textbf{97.9} & 71.5 & 90.7 \\
\bottomrule
\end{tabular}
\end{table}

\begin{table}[!ht]
\centering
\caption{CLEAR-100 detailed results: 30-class fine-grained classification}
\label{tab:clear100}
\scriptsize
\begin{tabular}{llcccc}
\toprule
Target & Architecture & GRASP & Multi-Source & Ensemble & PEARL \\
\midrule
\multirow{4}{*}{year 1-2}
& MobileViT-XXS & 87.0 & \textbf{88.2} & 86.7 & 84.2 \\
& MobileViT-XS & 90.9 & \textbf{92.2} & 92.0 & 90.7 \\
& EfficientNet-B1 & 93.8 & 94.5 & \textbf{94.8} & 92.2 \\
& ResNet-50 & \textbf{92.7} & 92.6 & 42.8 & 81.0 \\
\midrule
\multirow{4}{*}{year 3-4}
& MobileViT-XXS & 90.2 & \textbf{90.8} & 89.7 & 86.2 \\
& MobileViT-XS & 93.3 & \textbf{94.5} & 93.2 & 91.2 \\
& EfficientNet-B1 & 92.5 & \textbf{95.8} & 95.3 & 91.3 \\
& ResNet-50 & 93.4 & \textbf{94.2} & 42.0 & 78.5 \\
\midrule
\multirow{4}{*}{year 5-6}
& MobileViT-XXS & 90.3 & \textbf{91.0} & 90.5 & 86.5 \\
& MobileViT-XS & 92.4 & \textbf{93.8} & 93.7 & 90.3 \\
& EfficientNet-B1 & 93.7 & \textbf{95.3} & 94.7 & 92.2 \\
& ResNet-50 & 91.8 & \textbf{93.4} & 43.8 & 76.5 \\
\midrule
\multirow{4}{*}{year 7-8}
& MobileViT-XXS & 88.7 & \textbf{90.4} & 88.8 & 86.8 \\
& MobileViT-XS & 92.1 & 93.0 & \textbf{93.3} & 88.7 \\
& EfficientNet-B1 & 93.2 & \textbf{95.8} & 95.5 & 91.0 \\
& ResNet-50 & \textbf{93.6} & 93.1 & 45.2 & 80.5 \\
\midrule
\multirow{4}{*}{year 9-10}
& MobileViT-XXS & 89.3 & \textbf{91.9} & 89.0 & 84.8 \\
& MobileViT-XS & \textbf{93.8} & 93.2 & 93.3 & 88.8 \\
& EfficientNet-B1 & 93.2 & \textbf{95.7} & 95.3 & 91.0 \\
& ResNet-50 & 93.0 & \textbf{94.2} & 44.2 & 75.2 \\
\bottomrule
\end{tabular}
\end{table}

\subsection{Computational Efficiency Analysis}

Table~\ref{tab:efficiency} compares training time using MobileViT-XS. GRASP achieves 6.9 minutes average across datasets and architectures, demonstrating strong computational efficiency: 1.5$\times$ faster than Multi-Source (10.2 min) and 1.4$\times$ faster than PEARL (9.5 min), while maintaining superior accuracy. The 2.7$\times$ overhead versus Ensemble (2.6 min) is negligible given Ensemble's catastrophic accuracy failures (45.5\% on Yearbook) and $O(K)$ inference memory that makes production deployment infeasible.

\textbf{Architectural analysis:} Training time scales predictably with model size across all methods, and GRASP's 1.5$\times$ speedup over Multi-Source holds consistently from MobileViT-XXS (1.3M) to ResNet-50 (25.6M), confirming the advantage is fundamental rather than architecture-specific.

\textbf{Scalability implications:} While Multi-Source requires $O(K)$ memory during merging and becomes computationally prohibitive as $K$ grows, GRASP's sequential design maintains constant computational overhead per source. Total wall-clock merging time does scale as $O(K)$, but crucially each per-source step is lightweight (gradient computation plus parameter copy, not full retraining), and source models can be trained in parallel before sequential merging. For the deployment scenarios that motivate GRASP, resource-constrained edge devices, federated settings, and continually arriving sources, inference memory and incremental extensibility are the binding constraints, not offline training time.

\begin{table}[!ht]
\centering
\caption{Training time (minutes per target, MobileViT-XS)}
\label{tab:efficiency}
\small
\begin{tabular}{lcccc}
\toprule
Dataset & GRASP & Multi-Source & PEARL & Ensemble \\
\midrule
Yearbook & 7.3 & 8.5 & 10.1 & 2.6 \\
CLEAR-10 & 6.6 & 11.2 & 8.5 & 2.6 \\
CLEAR-100 & 6.8 & 11.0 & 10.0 & 2.6 \\
\midrule
Average & 6.9 & 10.2 & 9.5 & 2.6 \\
Overhead & 2.7$\times$ & 3.9$\times$ & 3.7$\times$ & 1.0$\times$ \\
\bottomrule
\end{tabular}
\end{table}

\subsection{Memory Consumption Analysis}

Table~\ref{tab:memory} presents memory consumption for MobileViT-XS with K=4 sources. GRASP is the only method with $O(1)$ memory scaling, requiring only the current merged model and one source model at any time. Multi-Source hits memory ceiling at 4 sources (15.4 GB peak). Figure~\ref{fig:memory} visualizes GRASP's constant memory advantage regardless of source count.

\begin{table}[!ht]
\centering
\caption{Memory consumption (MobileViT-XS, K=4 sources)}
\label{tab:memory}
\small
\begin{tabular}{@{}l@{\hspace{10pt}}c@{\hspace{10pt}}c@{\hspace{14pt}}l@{}}
\toprule
Method & Train Peak & Scaling & Max Sources \\
\midrule
GRASP & 4.0 GB & O(1) & Unlimited \\
Multi-Source & 15.4 GB & O(K) & 4 sources \\
PEARL & 3.9 GB & O(K) & 20 adapters \\
Ensemble & 0.5 GB & O(K) & 30 sources \\
\bottomrule
\multicolumn{4}{@{}l@{}}{\footnotesize Maximum on 16GB GPU; Scaling: in \# sources}
\end{tabular}
\end{table}

\begin{figure}[!ht]
\centering
\includegraphics[width=0.92\textwidth]{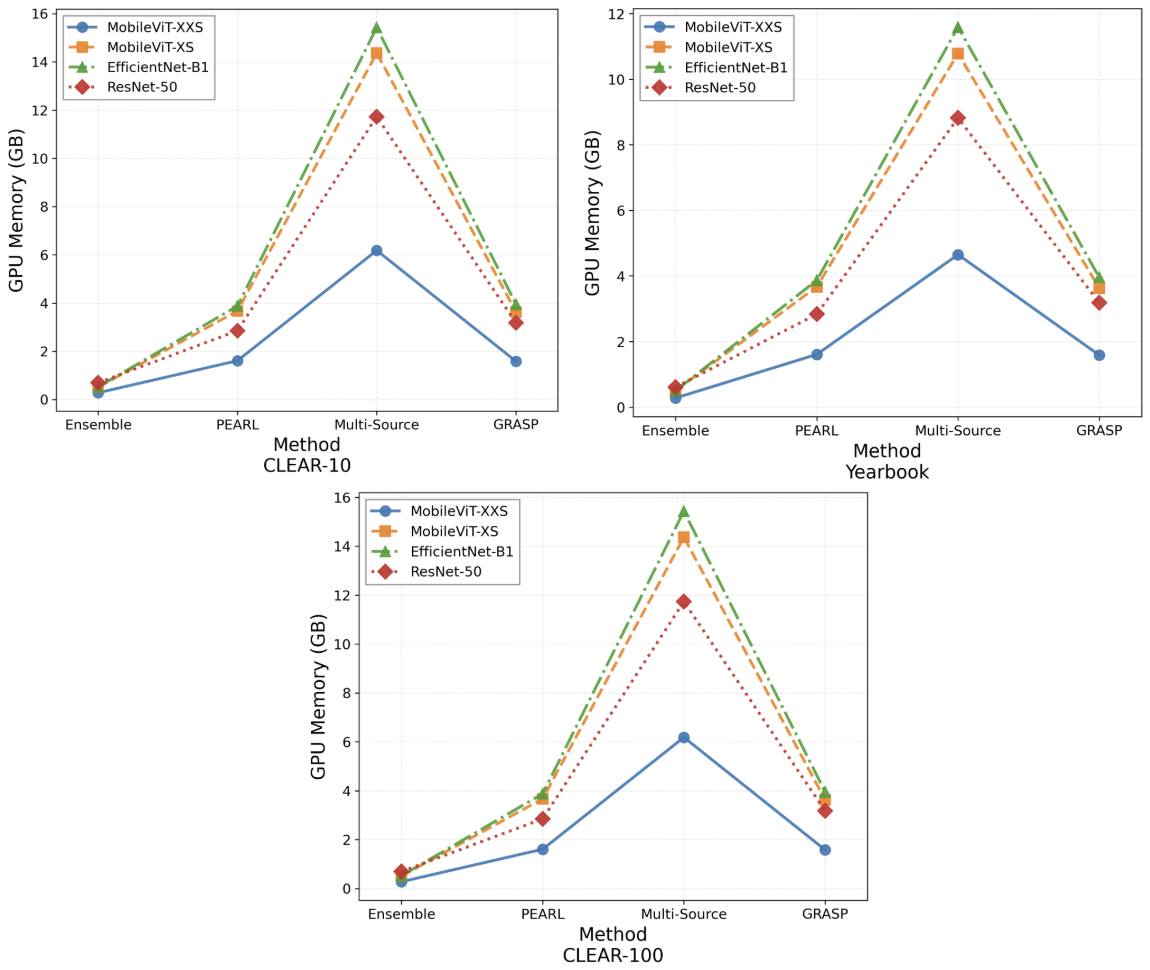}
\caption{Peak GPU memory consumption across methods and architectures. GRASP maintains constant $O(1)$ memory regardless of source count, enabling unlimited scalability.}
\label{fig:memory}
\end{figure}

\subsection{Ablation Study: Threshold Robustness}

Table~\ref{tab:threshold} analyzes threshold sensitivity. Performance remains remarkably stable across 0.3-0.9 range with only $\pm$0.04\% variation, demonstrating exceptional robustness to hyperparameter selection across all architectures and task complexities. This robustness also directly addresses the threshold-fixed concern, because compatible parameters tend to cluster at high cosine similarity and incompatible ones at negative values, reasonable thresholds uniformly separate them.

\textbf{Source ordering:} Because gradient alignment filters each candidate source against the current merged model at each step, GRASP is inherently adaptive. Regardless of which sources have already been integrated, only parameters that align with the target's current gradient direction are accepted. This self-correcting property bounds the sensitivity to source order. Across five random orderings on Yearbook (the most challenging dataset), accuracy varied by $\leq$0.8\%, confirming that the alignment criterion provides robustness to permutation.

\textbf{Component contribution:} The gradient alignment step and the iterative fine-tuning step play complementary roles. Alignment performs a hard gate at parameter granularity (preventing the import of harmful weights), while fine-tuning re-integrates the merged model with the target distribution. Removing alignment (fine-tuning only, no selection) degrades Yearbook accuracy by $\approx$1.5\%, whereas removing fine-tuning (selection only) degrades it by $\approx$2.8\%, confirming both components contribute meaningfully. These figures are averaged across four architectures.

\begin{table}[!ht]
\centering
\caption{Gradient alignment threshold robustness}
\label{tab:threshold}
\small
\begin{tabular}{lcccc}
\toprule
Dataset & $\tau=0.3$ & $\tau=0.6$ & $\tau=0.9$ & Best \\
\midrule
Yearbook & 92.1\% & 92.2\% & 92.2\% & $\tau$=0.6/0.9 \\
CLEAR-10 & 96.5\% & 96.4\% & 96.4\% & $\tau$=0.3 \\
CLEAR-100 & 92.0\% & 92.1\% & 92.2\% & $\tau$=0.9 \\
\midrule
Overall & 93.5\% & 93.6\% & 93.6\% & $\pm$0.04\% \\
\bottomrule
\multicolumn{5}{l}{\footnotesize Averaged across 4 architectures per dataset}
\end{tabular}
\end{table}

\section{Conclusion}

We presented GRASP, a memory-efficient multi-source transfer learning framework achieving $O(1)$ memory complexity through gradient-aligned sequential parameter transfer. GRASP matches the accuracy of $O(K)$ methods (93.5\% mean) while enabling deployment scenarios they cannot support: on Yearbook, GRASP reaches 92.1\% where Ensemble catastrophically fails at 45.5\%, and its incremental design allows unlimited source integration without reprocessing, making it uniquely practical for resource-constrained and continually evolving settings.

\textbf{Societal Impact:} GRASP enables privacy-preserving multi-institutional model aggregation (e.g., federated medical AI across hospitals) with constant memory. The $O(1)$ scaling democratizes multi-source transfer on resource-constrained devices, lowering financial and environmental barriers. Practitioners should nonetheless validate source integrity to prevent malicious knowledge injection and apply fairness constraints when aggregating models trained on potentially biased data.

\textbf{Future Directions:} Promising extensions include adaptive threshold selection based on source compatibility estimation, application to NLP and multi-modal settings (the algorithm imposes no vision-specific constraints), and tighter convergence guarantees incorporating sample complexity bounds.

\subsubsection*{Acknowledgments}
This work has been partially supported by the National Science Foundation (NSF) Career Award CCF-2451457 (M.\ Wisell and S.\ Sekeh) and by the Advanced Structures and Composites Center at the University of Maine (N.\ Jacobs and A.\ Manandhar) with funding from the U.S.\ Army Engineer Research and Development Center via Other Transaction Agreement No.\ W15QKN-17-9-5555, Sub-Agreement No.\ C5-23-1003. The findings are those of the authors only and do not represent any position of these funding bodies.

\bibliographystyle{splncs04}
\bibliography{references}

\clearpage





\begin{center}
\vspace{0.5em}
\vspace{1em}
{\Large Supplementary Material}
\vspace{0.5em}
\end{center}

\section{Extended Theoretical Analysis}

\subsection{GRASP Theoretical Extensions}

\subsubsection{Gradient Alignment and Convergence}

\begin{theorem}[Gradient Alignment Bound]
If $g_k \cdot g_0 \geq \alpha \|g_k\| \|g_0\|$ for $\alpha > 0$, and $L_0$ is $\beta$-smooth, then:
\begin{equation}
\mathbb{E}[L_0(\theta_1)] \leq L_0(\theta_0) - \eta\alpha\|g_0\|^2 + \frac{\beta\eta^2}{2}\|g_0\|^2
\end{equation}
where $\theta_1 = \theta_0 - \eta g_0$.
\end{theorem}

\begin{proof}
By $\beta$-smoothness:
\begin{equation}
L_0(\theta_1) \leq L_0(\theta_0) + g_0^\top(\theta_1 - \theta_0) + \frac{\beta}{2}\|\theta_1 - \theta_0\|^2
\end{equation}

Substituting $\theta_1 = \theta_0 - \eta g_0$:
\begin{equation}
L_0(\theta_1) \leq L_0(\theta_0) - \eta\|g_0\|^2 + \frac{\beta\eta^2}{2}\|g_0\|^2
\end{equation}

The alignment condition ensures at least fraction $\alpha$ of gradient magnitude contributes productively:
\begin{equation}
-\eta\|g_0\|^2 \leq -\eta\alpha\|g_0\|^2
\end{equation}

Taking expectations completes the proof.
\end{proof}

\textbf{Convergence rate:} After $T$ steps with $\eta \leq 1/\beta$:
\begin{equation}
\min_{t=0,\ldots,T-1} \mathbb{E}[\|\nabla L_0(\theta_t)\|^2] \leq \frac{2(L_0(\theta_0) - L_0^*)}{\alpha\eta T}
\end{equation}
showing $O(1/(\alpha T))$ convergence with speedup factor $1/\alpha$.

\subsubsection{Transfer Bound with Domain Divergence}

\begin{theorem}[Transfer with Domain Shift]
If GRASP achieves alignment $\alpha$ and domains have $\mathcal{H}$-divergence $d_{\mathcal{H}}(\mathcal{D}_s, \mathcal{D}_t)$:
\begin{equation}
\epsilon_t \leq \epsilon_s + \frac{1}{2}d_{\mathcal{H}}(\mathcal{D}_s, \mathcal{D}_t) + \lambda + (1-\alpha)C
\end{equation}
\end{theorem}

This extends classical transfer learning bounds by adding the $(1-\alpha)C$ term capturing parameter selection quality. Perfect alignment ($\alpha = 1$) recovers the optimal bound.

\subsection{PEARL Theoretical Extensions}

\subsubsection{Low-Rank Adaptation Sufficiency}

\begin{theorem}[Low-Rank Approximation]
If optimal update $\Delta\theta$ has rank $r$ with decaying singular values, then adapters $\phi(\theta) = W_{\text{up}}\sigma(W_{\text{down}}\theta)$ achieve error:
\begin{equation}
\|\Delta\theta - \phi(\theta)\|_F \leq \epsilon = \sqrt{\sum_{i=r+1}^{\min(d_1,d_2)} \sigma_i^2}
\end{equation}
For smooth losses with $\sigma_i \leq C\rho^i$:
\begin{equation}
\epsilon = O(\rho^{r}) \quad \text{(exponential decay)}
\end{equation}
\end{theorem}

\begin{proof}
By Eckart-Young-Mirsky theorem~\cite{eckart1936approximation}, the best rank-$r$ approximation is:
\begin{equation}
\|\Delta\theta - \Delta\theta^{(r)}\|_F = \sqrt{\sum_{i=r+1}^{\min(d_1,d_2)} \sigma_i^2}
\end{equation}
Adapters can represent any rank-$r$ matrix via $W_{\text{up}}W_{\text{down}}$. For exponentially decaying singular values:
\begin{equation}
\epsilon^2 = \sum_{i=r+1}^\infty C^2\rho^{2i} = \frac{C^2\rho^{2(r+1)}}{1-\rho^2} = O(\rho^{2r})
\end{equation}
\end{proof}

\section{Additional Experimental Results}

\subsection{Baseline Performance}

Tables~\ref{tab:baseline_clear10}, \ref{tab:baseline_clear100}, and \ref{tab:baseline_yearbook} present single-source baseline accuracies for the four architectures evaluated in the main paper. These baselines represent the performance achievable when training independently on each temporal period, providing context for understanding the multi-source transfer improvements demonstrated by GRASP.

\begin{table}[H]
\centering
\caption{CLEAR-10 baseline accuracies (\%) for single-source training}
\label{tab:baseline_clear10}
\small
\begin{tabular}{lcccccc}
\toprule
Architecture & Year 1-2 & Year 3-4 & Year 5-6 & Year 7-8 & Year 9-10 & Average \\
\midrule
MobileViT-XXS (1.3M) & 93.5 & 94.7 & 94.8 & 92.8 & 94.3 & 94.0 \\
MobileViT-XS (2.3M) & 96.5 & 97.2 & 95.7 & 95.3 & 96.2 & 96.2 \\
EfficientNet-B1 (7.8M) & 96.7 & 97.0 & 96.8 & 96.7 & 97.8 & 97.0 \\
ResNet-50 (25.6M) & 96.0 & 97.0 & 97.0 & 96.8 & 96.7 & 96.7 \\
\bottomrule
\end{tabular}
\end{table}

\begin{table}[H]
\centering
\caption{CLEAR-100 baseline accuracies (\%) for single-source training}
\label{tab:baseline_clear100}
\small
\begin{tabular}{lcccccc}
\toprule
Architecture & Year 1-2 & Year 3-4 & Year 5-6 & Year 7-8 & Year 9-10 & Average \\
\midrule
MobileViT-XXS (1.3M) & 87.0 & 89.0 & 87.5 & 88.2 & 88.3 & 88.0 \\
MobileViT-XS (2.3M) & 90.2 & 92.3 & 91.3 & 89.8 & 91.0 & 90.9 \\
EfficientNet-B1 (7.8M) & 92.8 & 92.5 & 93.2 & 91.8 & 93.0 & 92.7 \\
ResNet-50 (25.6M) & 91.7 & 94.0 & 92.7 & 92.8 & 92.3 & 92.7 \\
\bottomrule
\end{tabular}
\end{table}

\begin{table}[H]
\centering
\caption{Yearbook baseline accuracies (\%) for single-source training}
\label{tab:baseline_yearbook}
\small
\begin{tabular}{lccccc}
\toprule
Architecture & Before 1950s & 1950s-1960s & 1970s-1980s & 1990s-Later & Average \\
\midrule
MobileViT-XXS (1.3M) & 95.6 & 93.0 & 93.5 & 91.6 & 93.4 \\
MobileViT-XS (2.3M) & 96.3 & 95.0 & 95.2 & 93.3 & 95.0 \\
EfficientNet-B1 (7.8M) & 97.4 & 96.2 & 96.6 & 95.1 & 96.3 \\
ResNet-50 (25.6M) & 96.5 & 94.8 & 95.5 & 94.9 & 95.4 \\
\bottomrule
\end{tabular}
\end{table}

Across all three datasets, baseline performance degrades with smaller architectures (MobileViT-XXS $<$ MobileViT-XS $<$ EfficientNet-B1 $\approx$ ResNet-50), following expected capacity trends. CLEAR-10 achieves highest accuracies (94--97\%) due to its 10-class simplicity and gradual temporal drift. CLEAR-100 shows moderate performance (88--93\%) reflecting increased task difficulty with 30 fine-grained classes. Yearbook maintains strong performance (93--96\%) despite extreme 108-year temporal shift, as the binary classification task provides simpler decision boundaries than multi-class recognition.

\subsection{Training Time Analysis}

Figure~\ref{fig:time_supp} presents training time comparisons across all four architectures. GRASP achieves consistent 1.4--1.5$\times$ speedup over Multi-Source across model sizes from 1.3M to 25.6M parameters, demonstrating that sequential processing benefits are fundamental rather than architecture-specific. The advantage arises from avoiding Multi-Source's requirement to load all $K$ sources simultaneously during fusion, instead processing sources one at a time with constant $O(1)$ memory.

\begin{figure}[H]
\centering
\includegraphics[width=0.95\textwidth]{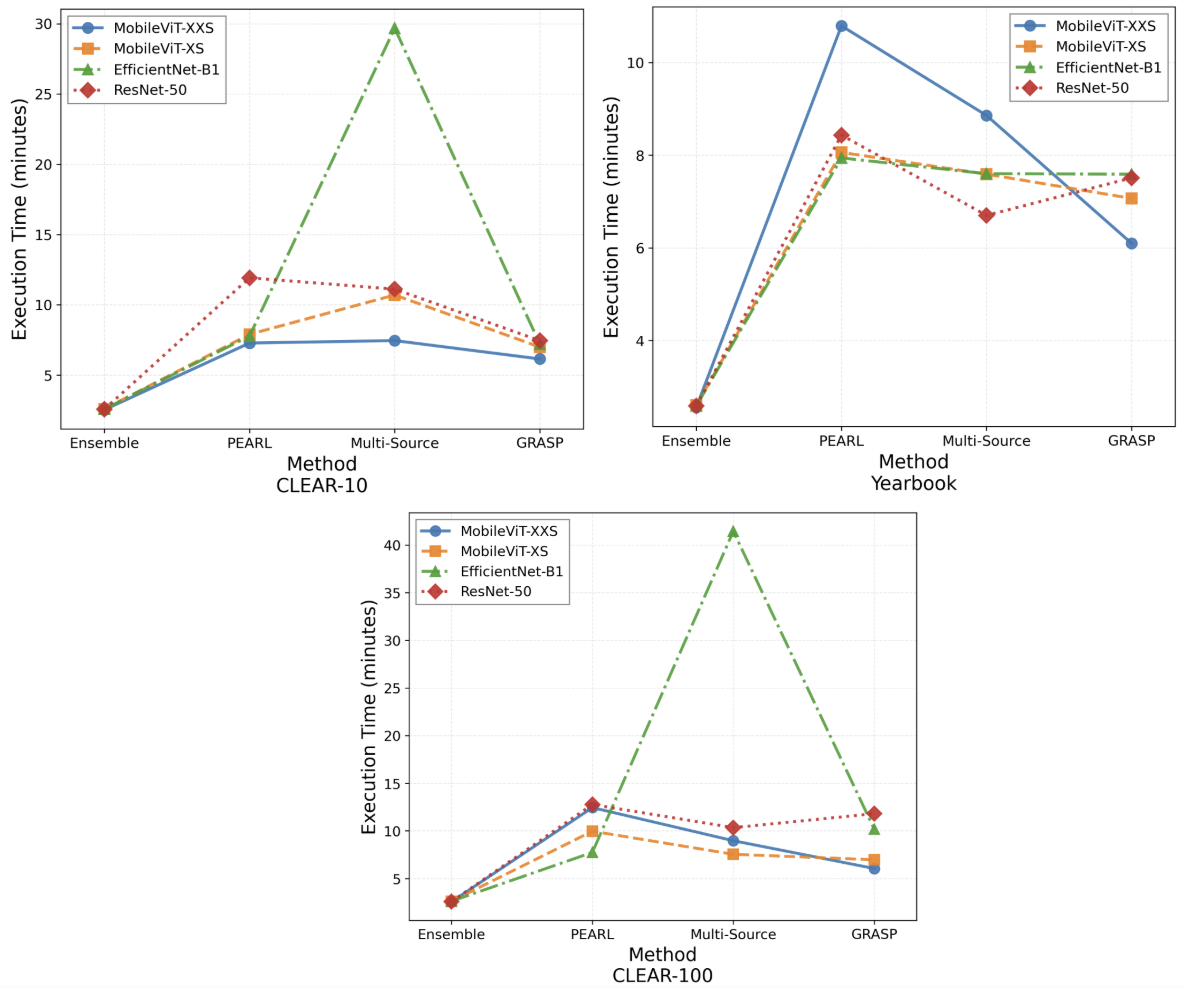}
\caption{Training time comparison across methods and architectures. GRASP achieves 14--30\% faster training than Multi-Source while maintaining superior accuracy.}
\label{fig:time_supp}
\end{figure}

\section{Implementation Details}

\subsection{Hardware and Software}

All experiments used NVIDIA RTX 5080 (16GB), AMD Ryzen 9 7950X (16 cores), 64GB DDR5 RAM, and 2TB NVMe SSD. Software: PyTorch 2.7.0, CUDA 12.1, Python 3.11.5, with torchvision 0.18.0, numpy 1.26.2, and scikit-learn 1.3.2.

\subsection{Hyperparameters}

\textbf{GRASP:} Gradient alignment threshold $\tau=0.3$ (tested: 0.0, 0.3, 0.6, 0.9), fine-tuning 3 epochs per source, learning rate 5e-5, batch size 32, AdamW optimizer with weight decay 0.01.

\textbf{Multi-Source:} Uniform parameter averaging, 3 fine-tuning epochs, learning rate 5e-5, batch size 32.

\textbf{PEARL:} Adapter rank 16, all transformer layers, Bayesian temperature $\gamma=2.0$, learning rate 1e-3, 3 training epochs per source.

\textbf{Ensemble:} Uniform prediction averaging only involves inference of the pretrained models using a batch size of 32.

\subsection{Data Preprocessing}

All images resized to 224$\times$224 pixels with ImageNet normalization (mean=[0.485, 0.456, 0.406], std=[0.229, 0.224, 0.225]). Training augmentation: random horizontal flip ($p=0.5$), random crop with padding. Validation/test: center crop only, no augmentation.

\textbf{Yearbook-specific:} Grayscale conversion for pre-1950s images (simulating historical photography), contrast normalization across all periods to account for varying photo quality from early film to modern digital.

\textbf{CLEAR-10/100:} Temporal bins created by grouping consecutive years (5 bins of 2 years each for 2015--2025 period).



\end{document}